\documentclass[submission,copyright,creativecommons]{eptcs}
\providecommand{\event}{ICLP 2026} 

\usepackage{iftex}
\usepackage{amsmath}
\usepackage{amssymb}
\usepackage{amsthm}
\usepackage[inline]{enumitem}
\usepackage{tikz-cd}
\usepackage{cancel}

\newtheorem{parametrization}{Parametrization}
\newtheorem{principle}{Principle}

\newtheorem{formalization}{Formalization}
\newtheorem{example}{Example}
\newtheorem{definition}{Definition}
\newtheorem{remark}{Remark}
\newtheorem{theorem}{Theorem}
\newtheorem{proposition}{Proposition}

\DeclareMathOperator{\sufficient}{sufficient}
\DeclareMathOperator{\constraint}{constraint}
\DeclareMathOperator{\necessary}{necessary}
\DeclareMathOperator{\MLN}{MLN}

\DeclareMathOperator{\pa}{pa}

\DeclareMathOperator{\head}{head}

\DeclareMathOperator{\body}{body}

\DeclareMathOperator{\graph}{graph}

\DeclareMathOperator{\Pa}{Pa}

\DeclareMathOperator{\LP}{LP}

\DeclareMathOperator{\explains}{explains}

\usepackage{graphicx}

\ifpdf
  \usepackage{underscore}         
  \usepackage[T1]{fontenc}        
\else
  \usepackage{breakurl}           
\fi

\title{How Rules Represent Causal Knowledge: \\ Causal Modeling
with Probabilistic Logic Programming}
\author{Kilian Rückschloß
\institute{University of Tübingen\\ Tübingen, Germany}
\email{kilian.rueckschloss@uni-tuebingen.de}
\and
Felix Weitkämper
\institute{German University Of Digital Science\\
Potsdam, Germany}
\email{felix.weitkaemper@german-uds.de}
}
\def\titlerunning{Causal Modeling with Probabilistic Logic Programming}
\def\authorrunning{Rückschloß and Weitkämper}
\begin{document}
\maketitle

\begin{abstract}
Pearl famously argues that causal knowledge enables the prediction of intervention effects. By contrast, purely descriptive knowledge supports only conclusions drawn from observations. His theory of causality, however, is developed exclusively within Bayesian networks and causal models. Consequently, it is largely restricted to acyclic causal relationships, and transferring its ideas to other formalisms risks misinterpretation or inconsistency.

This paper brings Pearl’s approach to causality into probabilistic logic programming (PLP). To this end, such programs are aligned with philosophical foundations established in prior work that do not rely on temporal notions; that is, all relevant events are assumed to occur simultaneously. A formal causal semantics for these programs, together with a notion of intervention and an implementation, is proposed. It is shown that this semantics coincides with the P-log semantics for stratified ProbLog programs, while the two may differ in the non-stratified case and for other PLP formalisms.
\end{abstract}

\section{Introduction}

Causality has been a central topic in philosophy for over two thousand years before Pearl~\cite{Causality} brought it into the mainstream of artificial intelligence. With Bayesian networks and causal models, he introduced representations of probability distributions that account for intervention effects. Since distributions alone support only conclusions drawn from observations, he argues that these representations encode causal knowledge that goes beyond the descriptive information contained in the distribution. As evaluating effects of interventions, like actions, is a primary goal of modeling, this view has led to the adoption of causal methods across various domains~\cite{ArifM22,GaoZWFHL24,WuPLZSLQLG24}.

\begin{example}
Smokers are more likely to carry matches, so observing matches suggests that a person smokes, whereas intervening by carrying matches does not force one to smoke. 

According to Pearl~\cite{Causality}, this distinction between observing and intervening requires causal rather than merely descriptive knowledge.
\label{example - smoking and matches}
\end{example}

While Pearl~\cite{Causality} studies causality primarily in terms of interventions, philosophical work more often focuses on explaining why events occur. To bridge this gap, Eelink et al.~\mbox{\cite{EelinkRW25}} draw on Aristotle’s \emph{Posterior Analytics}~\cite{Barnes} and propose a logical theory of causality. They argue that scientific knowledge, that is, justified true belief whose justifications follow the causal order, enables reasoning about interventions.

\begin{example}
Smoking causes people to carry matches, not the other way around. Since, in Example~\ref{example - smoking and matches}, carrying matches is thus not a cause of smoking, it does not make one smoke.
\label{example - smoking and matches - 2}
\end{example}

Building on prior work by Bochman~\cite{bochman}, Eelink et al.~\cite{EelinkRW25} adopt the view that causal explanations are given by composing rules of the form ``smoking parents \emph{cause} their children to smoke'' or ``smoking \emph{causes} carrying matches.'' Rückschloß and Weitkämper~\mbox{\cite{RueckschlossWeitkaemper:2025}} further argue that, in the absence of uncertainty, this makes rule-based languages in logic programming a natural representation of causal knowledge. 
Finally, Eelink et al.~\cite{EelinkRW25} also account for probabilistic causal rules such as ``smoking \emph{causes} carrying matches with probability~$0.8$'', while  extending logic programming with uncertainty gives rise to the field of probabilistic logic programming~\cite{PLP}.

\subsubsection*{Overview of Contributions}
This work connects the probabilistic logic programming frameworks $\LP^{\MLN}$~\cite{LPMLN} and ProbLog~\cite{ProbLog} with the theory of causality of Eelink et al.~\cite{EelinkRW25}. Examples~\ref{example - counterexample ProbLog}--\ref{example - counterexample LPMLN} show that non-stratified ProbLog and $\LP^{\MLN}$ programs require a causal semantics different from the standard ones~\cite{P-log,LPMLN}. To address this issue, Definitions~\ref{definition - causal systems} and~\ref{definition - causal semantics} transfer the formal causal semantics of Eelink et al.~\cite{EelinkRW25} to these frameworks. In this setting, Theorem~\ref{theorem - Consistency for Stratified ProbLog Programs} shows that stratified ProbLog programs under their standard semantics~\cite{P-log,ProbLog} admit a sound causal interpretation. Theorems~\ref{theorem - Consistency for Stratified ProbLog Programs - 1} and~\ref{theorem - causal irrelevance} establish that the proposed approach extends this interpretation while preserving key principles of causal reasoning from epistemology. Lastly, Section~\ref{sec: Implementation} presents an implementation of the proposed semantics.

\subsubsection*{Related Work}
Rückschloß and Weitkämper~\cite{fcm-semantics} use Clark completion~\cite{ClarkCompletion} to interpret acyclic ProbLog programs~\cite{ProbLog} as causal models in the sense of Pearl~\cite{Causality}, thereby transferring his notion of intervention to probabilistic logic programming. The present paper generalizes this approach to cyclic ProbLog and $\LP^{\MLN}$ programs~\cite{LPMLN} by building on the work of Eelink et al.~\cite{EelinkRW25}, which itself generalizes Pearl's account~\cite{Causality}.

Vennekens et al.~\cite{cplogic} also provide a causal interpretation of probabilistic logic programs. Their \emph{execution models} describe stochastic processes that represent a causal order through a temporal succession of events determined by the program rules. In contrast, the approach developed here is based on causal explanation rather than temporal notions and assumes that all relevant events occur simultaneously. 

\begin{example}
    Anna and Kilian are drinking tea when the bell rings. Each will answer the door provided the other does not. Let the propositions~$a$ and~$k$ stand for “Anna answers the door” and “Kilian answers the door,” respectively. The causal rule $k \Leftarrow \neg a$ expresses that Anna’s not answering causes Kilian to answer, and $a \Leftarrow \neg k$ expresses that Kilian’s not answering causes Anna to answer. Hence, the situation is modeled by the logic program
    \(
        \mathbf{P} := \{ a \Leftarrow \neg k,\; k \Leftarrow \neg a \}.
    \)

    According to Vennekens et al.~\cite{cplogic}, there is no execution model for $\mathbf{P}$, since it is unclear which of the two rules applies first and the outcome depends on that order. The causal semantics developed here does not account for temporal succession. Consequently, either Kilian opens the door, causing Anna to remain seated, or Anna opens the door, causing Kilian to remain seated, with each possibility occurring with probability one half.
    \label{example - tea drinking}
\end{example}

\section{Preliminaries and Research Context}

The discussion is restricted to finite probability spaces and employs standard terminology for random variables. Relevant background can be found in introductory texts, such as Chapter~5 of Michelucci~\cite{Michelucci24}. 

\subsection{Maximizing Entropy, LogLinear Models and Bayesian Networks}

Adopting the Bayesian perspective, probabilities~${\pi(A) \in [0,1]}$ represent a rational agent’s degree of belief in an event~$A$. After observing an event~$B$, beliefs are updated by forming \textbf{conditional probabilities}~\cite{WILLIAMSON2009493}: 
\[
\pi(A \mid B) := \pi(B)^{-1} \cdot \pi(A, B)
\]


Fix a sample space~$\Omega$ and events~$A_1, \dots, A_n \subseteq \Omega$ with assigned probabilities~${\pi(A_i) := \pi_i \in [0,1]}$ for~${1 \leq i \leq n}$. A rational agent assumes the distribution~$\pi$ on~$\Omega$ according to the following principle:

\begin{principle}[Maximum Entropy]
The probability distribution~$\pi$ on~$\Omega$ is chosen to maximize the \textbf{entropy}~\({
  \displaystyle  H(\pi) := \sum_{\omega \in \Omega} -\pi(\omega) \cdot \ln (\pi(\omega))
}\)
subject to the constraints~$\pi(A_i) = \pi_i$ for~${1 \leq i \leq n}$.
\label{principle - maximum entropy}
\end{principle}

\begin{remark}
Shannon~\cite{Shannon} introduces the entropy~$H(\pi)$ as a measure of missing information in a distribution~$\pi$. Given only the probabilities~$\pi(A_1), \ldots, \pi(A_n)$, a rational agent should make as few additional assumptions as possible, which is achieved by maximizing entropy.
\label{remark - maximum entropy}
\end{remark}

\begin{example}
    As a calculation shows, according to the maximum entropy principle (Principle~\ref{principle - maximum entropy}), in the absence of further information, a rational agent assumes that a given die is fair, since there is no reason to prefer one outcome over another.
\end{example}

This work employs the language of propositional logic to reason about Boolean random variables. 

Following Franks~\mbox{\cite{sep-logic-propositional}}, it adopts standard notation for propositions $p$, formulas $\phi$, and structures~$\omega$. A \emph{structure}~$\omega$ is identified with the set of propositions that are true in it. The term \emph{world} refers to a consistent set of literals that is maximal with respect to inclusion. Here, a \emph{literal}~$l$ is either a proposition~$p$ or its negation $\neg p$. Since worlds are in one-to-one correspondence with structures, the two terms are used interchangeably.
Finally, sets of Boolean random variables~$\mathfrak{P}$ are identified with propositional alphabets, and formulas~$\phi$ over~$\mathfrak{P}$ with the events given by their models~$\{ \omega : \omega \models \phi \}$.

\begin{example}
    In Example~\ref{example - tea drinking}, the alphabet $\mathfrak{P} := \{ a, k \}$ is introduced. The structure $\omega := \varnothing$ is identified with the world $\{ \neg a,\; \neg k \}$, and the literal $\neg a$ with the event $\{ \varnothing,\; \{ k \} \}$.
\end{example}



For the rest of this section, fix a finite set of Boolean random variables~$\mathfrak{P} := \{ p_1, \ldots, p_n \}$ forming a propositional alphabet. 
Suppose that the probabilities~$\pi(\phi_i) \in (0,1)$ of formulas~$\phi_i$ are given. In general, maximum-entropy distributions cannot be described directly in terms of the probabilities~$\pi(\phi_i)$. However, since~$\pi$ is determined by one value per formula~$\phi_i$, it is natural to seek a parametrization using a single real number~$w_i \in \mathbb{R}$ for each~$\phi_i$.

\begin{parametrization}[Berger et al.~\cite{Berger}]
There exist \textbf{degrees of certainty} ${w_i \in \mathbb{R}}$, ${1\leq i \leq n}$ such that to every world $\omega$ the maximum entropy distribution $\pi$ assigns the probability
\[  \displaystyle
    {\pi (\omega) :=  \left( {\displaystyle \sum_{\omega' \text{ world}} \exp \left( \sum_{\omega' \models \phi_i} w_i \right)} \right)^{-1} \cdot \; {\displaystyle \exp \left( \sum_{\omega \models \phi_i} w_i \right)} }.
\]
\label{parametrization - maximum entropy model}
\end{parametrization}

A \textbf{LogLinear model} of Richardson and Domingos~\cite{MLN} is a finite set~$\Phi$ consisting of \textbf{weighted constraints} $(w, \phi)$, where~${w \in \mathbb{R} \cup \{ \pm \infty \}}$ is a real weight and $\phi$ is a formula. A \textbf{possible world}~$\omega$ is a world that models each \textbf{hard constraint} $(\pm \infty,\phi) \in \Phi$, i.e.,~${\omega \models \phi}$ whenever~${(+ \infty, \phi) \in \Phi}$ and ${\omega \models \neg \phi}$ whenever~${(- \infty, \phi) \in \Phi}$. 

A world $\omega$ then receives the \textbf{weight}
    \[{
    \displaystyle
    w_{\Phi} (\omega) := w (\omega)  := \begin{cases} \displaystyle
    \prod_{\substack{(w, \phi) \in \Phi,~ w \not \in \{ \pm \infty \} \\ \omega \models \phi}} \exp(w), & \omega \; \text{possible world} \\
    0 & \text{otherwise}
    \end{cases}.
    }
    \]
and Parametrization \ref{parametrization - maximum entropy model} assigns to $\omega$ the \textbf{probability}
\begin{equation}
    { \displaystyle
    \pi_{\Phi}^{\MLN} (\omega) := \pi (\omega) := \left( \sum_{\omega' \text{ world}} w(\omega') \right)^{-1} \cdot w(\omega).
    }
    \label{equation - probability MLN}
\end{equation}

\begin{remark}
    The weighted constraints ${(w, \phi) \in \Phi}$ with ${w \in \mathbb{R}}$ lack an intuitive interpretation. Only hard constraints~${(\pm \infty, \phi)}$ enforce that $\phi$ or $\neg \phi$ necessarily holds.  

    Here, $\MLN$ abbreviates ``Markov Logic Networks'', the formalism introduced by Richardson and Domingos~\cite{MLN} that results from lifting LogLinear models to first-order logic.
\end{remark}

Following the idea in Remark~\ref{remark - maximum entropy}, Eelink et al.~\cite{EelinkRW25} view the maximization of entropy subject to the probabilities $\pi(A_1), \ldots, \pi(A_n)$ as a form of justification or deduction for the resulting distribution~$\pi$. Accordingly, they employ LogLinear models to extend descriptive probabilistic knowledge.

\begin{example}
    Assume that Kilian and Anna are equally likely to hear the doorbell in Example~\ref{example - tea drinking}. One might be tempted to represent the resulting situation with the \mbox{LogLinear model}
    \[
    \Phi := \{ (\ln 2, k \leftarrow \neg a),\; (\ln 2, a \leftarrow \neg k) \}.
    \]
    Denote the possible worlds by $\omega_1 := \emptyset$, $\omega_2 := \{a\}$, $\omega_3 := \{k\}$, and $\omega_4 := \{a,k\}$. Then the corresponding weights are $w(\omega_1)=1$ and $w(\omega_2)=w(\omega_3)=w(\omega_4)=4$, yielding the probabilities $\pi(\omega_1)=\frac{1}{13}$ and $\pi(\omega_2)=\pi(\omega_3)=\pi(\omega_4)=\frac{4}{13}$. In particular, both Anna and Kilian answer the door together with probability $\frac{4}{13}$, contradicting the intuition in Example~\ref{example - tea drinking}. 
\label{example - LogLinear models}
\end{example}


Fix a set of random variables $\textbf{V}$. According to Pearl~\cite{Causality}, a \textbf{causal structure} on  $\textbf{V}$ is a directed acyclic graph $G$, i.e.~a partial order, on $\textbf{V}$. A variable~${V_1 \in \textbf{V}}$ is a \textbf{cause} of $V_2 \in \textbf{V}$ or $V_2$ is an \textbf{effect} of $V_1$ if there is a directed path from $V_1$ to $V_2$ in $G$. It is a \textbf{direct cause} of~$V_2$ if the edge~${V_1 \rightarrow V_2}$ exists in $G$, i.e. if and only if the node~${V_1 \in \Pa(V_2)}$ lies in the set~$\Pa(V_2)$ of \textbf{parents} of~$V_2$. 

\begin{example}
Cloudy weather, denoted $c$, causes rain, denoted $r$, which causes a road to be wet, denoted~$w$. This situation is described by the  causal structure: $c \rightarrow r \rightarrow w$.
\label{example - causal structure}
\end{example}

A \textbf{Bayesian network} $\textbf{BN} := (G, \pi(\cdot \vert \pa(\cdot)))$ on~$\textbf{V}$ consists of a causal structure~$G$ and the conditional probabilities~${\pi (v \vert \pa (V)) \in [0,1]}$ of the possible values $v$ of the random variables $V \in \textbf{V}$ conditioned on value assignments of their direct causes~$\pa(V)$. 

\begin{example}
    The causal structure $G$ from Example~\ref{example - causal structure}, together with the following probabilities, gives rise to a Bayesian network~${\textbf{BN} := (G, \pi(\cdot \vert \pa (\cdot)))}$:
    \begin{align}
        & \pi (c) := 0.5 
        && \pi (r \vert c) = 0.6 && 
         \pi (r \vert \neg c) = 0  
         \label{equation -  Bayesian network parameters}\\
        & \pi (w \vert r) = 0.9
        && \pi (w \vert \neg r) = 0.1.
        \label{equation -  extended Bayesian network parameters}
    \end{align}
    \label{example - Bayesian network}
\end{example}

 Given the information encoded in the Bayesian network $\mathbf{BN}$, Williamson~\cite{Williamson2001} maintains that the following principle of causal irrelevance, as formalized below, should be satisfied.

\begin{principle}[Causal Irrelevance]
    Accounting for additional unobserved effects should not alter beliefs.
    \label{principle - causal irrelevance}
\end{principle}

\begin{formalization}[Causal Irrelevance]
    Suppose that the Bayesian network $\mathbf{BN}^+ := (G^+, \pi^{+}(\cdot \mid \pa^+(\cdot)))$ extends $\mathbf{BN}$ by nodes ${V^+ \notin \mathbf{V}}$ that are not causes of any variable in $\mathbf{V}$. According to Principle~\ref{principle - causal irrelevance}, the networks $\mathbf{BN}$ and $\mathbf{BN}^+$ induce the same marginal distribution over $\mathbf{V}$.  
    \label{formalization - causal irrelevance - BN}
\end{formalization}

\begin{example}
    In Example~\ref{example - Bayesian network}, this means that the network $\mathbf{BN}$ and the network with structure $c \rightarrow r$ and Probabilities~(\ref{equation -  Bayesian network parameters}) induce the same distribution over~$\{c,r\}$. A calculation shows that this is not the case if entropy is maximized under Constraints~(\ref{equation -  Bayesian network parameters}) and~(\ref{equation - extended Bayesian network parameters}).
    \label{example - causal irrelevance}
\end{example}

Hence, Principles~\ref{principle - maximum entropy} and~\ref{principle - causal irrelevance} are in tension. Williamson~\cite{Williamson2001} resolves this apparent conflict by maximizing entropy greedily along the causal order. He shows that this yields the semantics of Pearl~\mbox{\cite{Causality}}. Consequently, the network~$\mathbf{BN}$ assigns to a value assignment~$\mathbf{v}$ on~$\mathbf{V}$ the probability
\begin{equation}
    \pi_{\mathbf{BN}}^{BN}(\mathbf{v}) := \pi(\mathbf{v}) :=
    \prod_{i=1}^k \pi\bigl(\mathbf{v}(V_i) \mid \pa(V_i)\bigr)
    \text{, where $\pa(V_i) := \mathbf{v}\vert_{\Pa(V_i)}$ for $1 \leq i \leq k$.}
    \label{equation - Bayesian network semantics}
\end{equation}

\begin{example}
     In Example \ref{example - Bayesian network}, one finds that
    $
        {\pi(\textit{cloudy}, \textit{rain}, \textit{wet}) = 
        0.5 \cdot 0.6  \cdot 0.9 = 0.27}
    $.
    \label{example - semantics of Bayesian networks}
\end{example}

To obtain an analogue of scientific knowledge in the sense of Aristotle's \emph{Posterior Analytics}~\cite{Barnes}, Eelink et al.~\cite{EelinkRW25} hold that justification—namely, entropy maximization—must respect the causal order. Hence, following Williamson~\cite{Williamson2001}, they argue that the semantics provided by Equation~(\ref{equation - Bayesian network semantics}) yields a probabilistic counterpart to this notion of knowledge, thereby explaining the Bayesian network's ability to reason about intervention effects. Here, this argument is applied to probabilistic logic programming.

\subsection{Abductive Logic Programming as a Representation for Causal Knowledge}

Eelink et al.~\cite{EelinkRW25} argue that causal explanations are composed of causal rules. According to Rückschloß and Weitkämper \cite{RueckschlossWeitkaemper:2025}, this makes abductive logic programming a natural formalism to represent causal knowledge in the absence of uncertainty.

A \textbf{logic program}~$\mathbf{P}$ is a set of \textbf{clauses}~$C$, that is, expressions of the form
\({
    h \Leftarrow b_1 \land (b_2 \land (\ldots \land b_n)\ldots),
}\)
also written as $h \Leftarrow b_1 \land \ldots \land b_n$, $h \Leftarrow b_1, \ldots, b_n$, or $\head(C) \Leftarrow \body(C)$. Here, $\head(C) := h \in \mathfrak{P}$ is an atom, called the \textbf{head} of~$C$, and $\body(C) := \{ b_1, \ldots, b_n \}$ is a finite set of literals, called the \textbf{body} of~$C$. If~${C = (h \Leftarrow \top)}$, that is, $\body(C) = \emptyset$, then $C$ is a \textbf{fact} and is simply written as~$h$.

The \textbf{(signed) dependence graph} $G(\mathbf{P})$ of~$\mathbf{P}$ is the directed graph over the alphabet~$\mathfrak{P}$ defined as follows:  
there is an edge $p \rightarrow q$ if and only if there exists a clause $C \in \mathbf{P}$ such that $\head(C) = q$ and~${\body(C) \cap \{ p, \neg p \} \neq \emptyset}$. The edge is \textbf{negative}, denoted $p \stackrel{-}{\rightarrow} q$, if $\neg p \in \body(C)$ and \textbf{positive}, denoted~${p \stackrel{+}{\rightarrow} q}$, if $p \in \body(C)$. Note that an edge may be both negative and positive simultaneously. A \textbf{cycle} in $G(\mathbf{P})$ is a finite alternating sequence of nodes and edges~\({
    q \rightarrow p_1 \rightarrow p_2 \rightarrow \dots \rightarrow p_n \rightarrow q
}\)
that begins and ends at the same node~$q$.  

The program~$\mathbf{P}$ is \textbf{stratified} if its dependence graph does not contain a cycle with a negative edge.

\begin{example}
    The logic program $\textbf{P}$ in Example \ref{example - tea drinking} is not stratified, it has the dependence graph
\begin{tikzcd}
a \arrow[r, "-"', bend right] & k \arrow[l, "-"', bend right]
\end{tikzcd}.
\end{example}


Rückschloß and Weitkämper~\cite{RueckschlossWeitkaemper:2025} give a causal interpretation to logic programming clauses. 
In this view, a clause ${h \Leftarrow b_1 \land \dots \land b_n}$ is 
interpreted as “$b_1 \land \dots \land b_n$ \emph{causes}~$h$” and called a \textbf{causal rule}. 
Following Bochman~\cite{bochman} and Eelink et al.~\mbox{\cite{EelinkRW25}}, these elementary causal rules are 
composed into a relation of \textbf{causal explainability}, denoted by~$(\Rrightarrow)/2$ or $(\Rrightarrow_{\textbf{P}})/2$, which is defined inductively from the rules in~$\textbf{P}$:
\begin{enumerate}
    \item[] \textbf{Causal Rules}: If $\lambda \Rightarrow l \in \mathbf{P}$, then $\lambda \Rrightarrow_{\mathbf{P}} l$. 
    \item[] \textbf{Monotonicity}: If $\lambda \Rrightarrow_{\mathbf{P}} l$, then $\lambda \cup \lambda' \Rrightarrow_{\mathbf{P}} l$ for any sets of literals $\lambda$, $\lambda'$ and literals $l$, $l'$. 
    \item[] \textbf{Cut}: If $\lambda' \Rrightarrow_{\mathbf{P}} l$ and $\lambda \cup \{ l \} \Rrightarrow_{\mathbf{P}} l'$, then $\lambda \cup \lambda' \Rrightarrow_{\mathbf{P}} l'$ for any sets of literals $\lambda$, $\lambda'$ and literals $l$, $l'$.
    \item[] \textbf{Contradiction}: $\{ p, \neg p \} \Rrightarrow_{\mathbf{P}} l$ for any proposition $p \in \mathfrak{P}$ and any literal $l$. 
\end{enumerate}
If $\lambda \Rrightarrow_{\textbf{P}} l$, it is said that $\lambda$ \textbf{explains} $l$.

\begin{remark}
    Like provability, $\vdash/2$, explainability, $\Rrightarrow/2$, defines a binary relation on formulas rather than a new logical connective. In particular, expressions of the form $\phi \Rrightarrow \psi \Rrightarrow \rho$ are meaningless.
\end{remark}

Eelink et al.~\cite{EelinkRW25} hold that the relation of causal explainability formalizes causal explanations as given by the following principle.

\begin{principle}[Causal Explanation]
    Causal explanations are deductions that respect the causal order.
    \label{principle - consistency with logical reasoning}
\end{principle}

The meaning of facts as propositions directly caused by truth itself is unclear;  
however, explanations that are not ultimately grounded in facts necessarily lead to cyclic arguments or infinite regress. 
Following Aristotle's \textit{Posterior Analytics}~\cite{Barnes}, Eelink et al.~\cite{EelinkRW25} hold that such arguments are not explanatory and commit to the following principle. 

\begin{principle}[Causal Foundation]
    \textit{Causal explanations} originate from external premises, whose explanations lie beyond a given scope. 
    \label{principle - causal foundation}
\end{principle}

Kakas and Mancarella~\cite{KakasM90} introduce abductive logic programming to explain observations. An \textbf{abductive logic program}~$\mathcal{P} := (\textbf{P}, \mathfrak{A})$ 
consists of a logic program~$\textbf{P}$ and a set of \textbf{abducibles}~${\mathfrak{A} \subseteq \mathfrak{P}}$  
such that no abducible~$u \in \mathfrak{A}$ is a clause head in~$\textbf{P}$. The program $\mathcal{P}$ has the dependence graph ${G(\mathcal{P}) := G(\textbf{P})}$ and is \textbf{stratified} if $\textbf{P}$ is.  Rückschloß and Weitkämper~\cite{RueckschlossWeitkaemper:2025} use these programs to represent causal knowledge as given by the causal rules in the logic program $\textbf{P}$ and the external premises in $\mathfrak{A}$.


\begin{example}
    Extending Example~\ref{example - tea drinking} by the propositions $ha$ and $hk$, denoting the events that Anna and Kilian hear the bell, respectively, yields the logic program 
    \(
        \mathbf{P} := \{ a \Leftarrow ha, \neg k,\;  k \Leftarrow hk, \neg a \}.
    \)
    Choosing the abducibles $\mathfrak{A} := \{ ha,\; hk \}$ yields the non-stratified abductive logic program $\mathcal{P} := (\mathbf{P}, \mathfrak{A})$.
    \label{example - non-stratified abductive logic program}
\end{example}

To derive a formal semantics that corresponds to causal explanation and scientific knowledge, Eelink et al.~\cite{EelinkRW25} further commit to the following principles from epistemology.

\begin{principle}[Natural Necessity]
    ``\ldots given the existence of the cause, the effect must necessarily follow.'' (Thomas Aquinas: \emph{Summa Contra Gentiles} II:35.4; translation by Anderson~\cite{Anderson})
    \label{principle - Aquinas}
\end{principle}  

\begin{principle}[Sufficient Causation]
    ``\ldots there is nothing without a reason, or no effect without a cause.''  
    (Gottfried Wilhelm Leibniz: \emph{First Truths}; translation by Loemker~\mbox{\cite{Leibniz}}, p.~268)
    \label{principle - Leibniz}
\end{principle}

\begin{principle}[Default Negation]
    Without information to the contrary, statements are considered false.
    \label{principle - default negation}
\end{principle}

Hereby, Principles~\ref{principle - Aquinas} and \ref{principle - Leibniz} are formalized as follows.

The \textbf{constraint content} of a clause~${C}$ is the implication 
\({
\constraint(C) := \head (C) \leftarrow \bigwedge \body (C).
}\)
For a logic program $\textbf{P}$, the \textbf{constraint content} is defined to be
   $
   {\constraint(\textbf{P}) := \{ \constraint(C) \text{ : } C \in \textbf{P} \}}
$. 
The \textbf{explanatory content} of $\textbf{P}$ for a world $\omega$ is the logic program
$
\textbf{P} \vert_\omega := \{ C \in \textbf{P} \text{ : } \omega \models \constraint (C) \}
$.

For an abductive logic program $\mathcal{P} := (\textbf{P}, \mathfrak{A})$,
the \textbf{constraint content} is
$\constraint(\mathcal{P}) := \constraint(\textbf{P})$
and the \textbf{explanatory content} is
$\mathcal{P} \vert_{\omega} := (\textbf{P} \vert_{\omega}, \mathfrak{A})$.
According to causal foundation (Principle~\ref{principle - causal foundation}) and default negation (Principle~\ref{principle - default negation}),
a proposition $p \in \mathfrak{P}$ is \textbf{explainable} in a world~$\omega$,
written $\omega \models \explains(p)$, if either $\omega \cap \mathfrak{A} \Rrightarrow p$,
$p \in \mathfrak{A}$, or~${\neg p \in \omega}$.

\begin{formalization}[Natural Necessity]
    A world $\omega$ satisfies natural necessity (Principle \ref{principle - Aquinas}) with respect to~$\mathcal{P}$ if $\omega \models \constraint(\mathcal{P})$.
    \label{formalization - necessity - LP}
\end{formalization}

\begin{formalization}[Causal Sufficiency]
    A world $\omega$ satisfies causal suffciency (Principle \ref{principle - Leibniz}) with respect to~$\mathcal{P}$ if all propositions are explainable, i.e.,
    ${\omega \models \explains (p)}$ for all propositions $p \in \mathfrak{P}$.
    \label{formalization - sufficiency - LP}
\end{formalization}

\begin{example}
    In Example \ref{example - non-stratified abductive logic program}, the world $\omega := \{ ha \}$ satisfies causal sufficiency (Principle \ref{principle - Leibniz}), but not natural necessity (Principle \ref{principle - Aquinas}). 
\end{example}
  
According to Formalizations \ref{formalization - necessity - LP} and \ref{formalization - sufficiency - LP}, the event that $\mathcal{P}$ satisfies \textbf{(natural) necessity} is the set
$$
\necessary (\mathcal{P}) := \{ \omega \text{ world: } \omega \models \constraint (\mathcal{P}) \}.
$$
and the event that $\mathcal{P}$ is \textbf{(causally) sufficient} is the set 
    $$
    \sufficient (\mathcal{P}) := \{ \omega \text{ world} \mid \forall\;p \in \mathfrak{P}:\; \omega \models \explains (p)   \}.
    $$

A \textbf{causally founded world} or \textbf{model} of $\mathcal{P}$ is a world~${\omega}$ that satisfies natural necessity (Principle \ref{principle - Aquinas}) and causal sufficiency (Principle \ref{principle - Leibniz}) as formalized above, i.e.,
\(
\omega \in \necessary (\mathcal{P}) \cap \sufficient (\mathcal{P}).
\)

In Theorem~4.2, Rückschloß and Weitkämper~\cite{RueckschlossWeitkaemper:2025} prove that every 
causally founded world of an abductive logic program~$\mathcal{P}$ is a stable model in the sense of Gelfond and Lifschitz~\cite{StableModelSemantics}, and vice versa. Consequently, one obtains the following result for stratified programs.

\begin{theorem}[Gelfond and Lifschitz~\cite{StableModelSemantics}]
    Let $\mathcal{P} := (\mathbf{P}, \mathfrak{A})$ be a stratified abductive logic program. For any subset of abducibles $\epsilon \subseteq \mathfrak{A}$, there exists a unique model $\omega$ such that $\omega \cap \mathfrak{A} = \epsilon$.~$\square$
    \label{theorem - unique stable models for stratified programs}
\end{theorem}

\begin{example}
    The program in Example \ref{example - non-stratified abductive logic program} has models $\omega_1 = \emptyset$, ${\omega_2 = \{ ha, a\}}$, ${\omega_3 = \{ hk, k\}}$, ${\omega_4 = \{ ha, hk, a\}}$ and $\omega_5 = \{ ha, hk , k\}$.
    \label{example - stable models of abductive logic program}
\end{example}

Following Rückschloß and Weitkämper~\cite{fcm-semantics},  Rückschloß and Weitkämper~\cite{RueckschlossWeitkaemper:2025} employ Clark completion~\cite{ClarkCompletion} to align abductive logic programs with the causal models of Pearl~\cite{Causality}. In this way, they transfer Pearl's notion of intervention to abductive logic programming.

To represent the intervention of forcing atoms in $\textbf{I} \subseteq \mathfrak{P} \setminus \mathfrak{A}$ to attain values according to assignment~$\textbf{i}$, the \textbf{modified program}~${
\mathcal{P}_\textbf{i} := (\textbf{P}_{\textbf{i}}, \mathfrak{A})}$ is obtained from $\mathcal{P} := (\textbf{P}, \mathfrak{A})$ by the following modifications:

\begin{enumerate*}
    \item Remove all clauses $C$ from $\textbf{P}$ for which $\textrm{head}(C) \in \textbf{I}$.~~~~
    \item Add a fact $p$ to $\textbf{P}_{\textbf{i}}$ whenever ${p \in \textbf{i}}$. 
\end{enumerate*}

\begin{example}
    Assume it is Halloween in Example~\ref{example - non-stratified abductive logic program}. Martina and Lukas ring the bell to ask for candy. Unfortunately, Anna and Kilian have only one bar of chocolate. If Kilian answers the door, he is grumpy and nobody receives the chocolate. If Anna answers the door, they receive the chocolate. In this case, Martina eats it if Lukas does not, and vice versa.

    Let $l$ and $m$ denote that Lukas and Martina eat the chocolate, respectively. The situation is modeled by the abductive logic program $\mathcal{P} := (\mathbf{P}, \mathfrak{A})$, where $\mathfrak{A} := \{ ha, hk \}$ and 
    \[
        \mathbf{P} := \{l \Leftarrow \neg m , a,\; m \Leftarrow \neg l , a,\; a \Leftarrow ha, \neg k,\; k \Leftarrow hk, \neg a \}.
    \] 

    If Martina does not feel well and decides not to eat any chocolate, this yields the modified program~${\mathcal{P}_{\neg m} := (\mathbf{P}_{\neg m}, \mathfrak{A})}$, where
    \(
        \mathbf{P}_{\neg m} := \{l \Leftarrow \neg m , a,\; \xcancel{m \Leftarrow \neg l , a},\; a \Leftarrow ha, \neg k,\; k \Leftarrow hk, \neg a \}.
    \)
    \label{example - intervention in abductive logic programs}
\end{example}

As illustrated in Examples \ref{example - smoking and matches} and \ref{example - smoking and matches - 2}, Rückschloß and Weitkämper~\cite{RueckschlossWeitkaemper:2025} argue that Program $\mathcal{P}$ correctly predicts intervention effects if the following principle of non-interference is satisfied.

\begin{principle}[Non-Interference]
     The impact of interventions is restricted to the causal direction. 
     \label{principle - non-interference}
\end{principle}

\begin{example}
    In Example~\ref{example - intervention in abductive logic programs}, Principle \ref{principle - non-interference} means that Martina not eating the chocolate does not affect who answers the door.
\end{example}

In Theorem 4.3, Rückschloß and Weitkämper~\cite{RueckschlossWeitkaemper:2025} argue that non-interference (Principle~\ref{principle - non-interference}) follows from causal irrelevance (Principle~\ref{principle - causal irrelevance}). In summary, this implies that, under the stable model semantics~\cite{StableModelSemantics}, stratified abductive logic programs can represent scientific knowledge in the sense of Aristotle's \emph{Posterior Analytics}~\cite{Barnes}, thereby enabling reasoning about intervention effects. 

\subsection{Probabilistic Logic Programming}
    
This work extends the causal interpretation and notion of intervention provided by Rückschloß and Weitkämper~\cite{RueckschlossWeitkaemper:2025} to probabilistic extensions of logic programming. Hereby, the focus lies on ProbLog programs as introduced by Kimmig et al.~\cite{ProbLog} and the $\LP^{\mathsf{MLN}}$ programs of Lee and Wang~\cite{LPMLN}.

The language $\LP^{\MLN}$ is inspired by the LogLinear models of Richardson and Domingos~\mbox{\cite{MLN}}. Lee and Wang~\cite{LPMLN} augment logic programs with degrees of certainty as introduced in Parametrization~\ref{parametrization - maximum entropy model}.

A \textbf{weighted clause} $(w, C)$ consists of a clause~$C$ and a weight~$w \in \mathbb{R} \cup \{ \infty \}$. Its \textbf{constraint content} is the weighted constraint
\(
\constraint(w,C) := (w, \constraint(C)) 
\).

An \textbf{$\mathbf{LP}^{\mathbf{MLN}}$ program}~$\mathbb{P}$ is a finite set of weighted clauses. Its \textbf{constraint content} is the LogLinear model
\(
\constraint(\mathbb{P}) := \{ \constraint(w,C) \mid (w,C) \in \mathbb{P} \},
\)
its \textbf{explanatory content} for a world~$\omega$ is the logic program
\(
\mathbb{P} \vert_{\omega} := \{ C \mid \exists w \ ((w,C) \in \mathbb{P} \land \omega \models \constraint(C)) \},
\)
and its \textbf{underlying logic program} is
\(
 \LP(\mathbb{P}) := \{ C \mid (w, C) \in \mathbb{P} \}.
\)
Its \textbf{dependence graph} is $G(\mathbb{P}) := G(\LP(\mathbb{P}))$, and $\mathbb{P}$ is \textbf{stratified} if $\LP(\mathbb{P})$ is.

A structure~$\omega$ is a \textbf{model} of~$\mathbb{P}$ if $\omega$ is (i) a model of the explanatory content $\mathbb{P}\vert_{\omega}$, and (ii) a possible world of the constraint content $\constraint(\mathbb{P})$. Inspired by Parametrization~\ref{parametrization - maximum entropy model}, a world~$\omega$ has \textbf{weight}
\[ \displaystyle
w_{\mathbb{P}} (\omega) := w(\omega) := 
\begin{cases}
\displaystyle
\prod_{\substack{(w,C) \in \mathbb{P}, \; w \neq \infty, \\ \omega \models \constraint(C)}} \exp(w), & \omega \text{ model of } \mathbb{P} \\
0, & \text{otherwise}
\end{cases}.
\]
and 
the \textbf{probability} $\pi_{\mathbb{P}}^{\LP^{\MLN}}(\omega)$ as given by Equation (\ref{equation - probability MLN}).

\begin{remark}
     In $\LP^{\MLN}$, $\LP$ stands for ``logic program'' and~$\MLN$ for ``Markov logic network''~\cite{MLN}. 
\end{remark}

\begin{example}
     Let the $\LP^{\MLN}$ program ${\textbf{P} := \{ (\ln(2), k \Leftarrow \neg a ), \; (\ln (2), a \Leftarrow \neg k )  \}}$ model the situation in Example \ref{example - LogLinear models}. The program $\textbf{P}$ has the stable models $\omega_1$ - $\omega_3$ with weights $w (\omega_1)  = 1$ and $w (\omega_2) = w (\omega_3) = 4$ and probabilities $\pi (\omega_1) = \frac{1}{9}$ and $\pi (\omega_2) = \pi (\omega_3) = \frac{4}{9}$. As desired there is a zero probability for Anna and Kilian to answer the door together. 
    \label{example - LPMLN programs}
\end{example}

The language ProbLog realizes probabilistic logic programming under the distribution semantics of Poole~\cite{distributionsemantics1} and Sato~\cite{distributionsemantics}. They assign probabilities to subprograms~$\textbf{P}' \subseteq \textbf{P}$ by randomly selecting clauses from a logic program~$\textbf{P}$. The resulting distribution on subprograms is then extended to a distribution over their models~\cite{StableModelSemantics}.  Building on this idea, Kimmig et al.~\cite{ProbLog} propose ProbLog and restrict randomness to abducibles or external propositions.

A \textbf{ProbLog program}~\({
\mathbb{P} := (\textbf{P}, \mathfrak{A}, \pi) \equiv (\mathcal{P}, \pi)
}\)
consists of an \textbf{underlying abductive logic program} \({
\mathcal{P} := (\textbf{P}, \mathfrak{A})
}\)
and a map $\pi$ assigning to each \textbf{error term} $u \in \mathfrak{A}$ a probability $\pi(u) \in [0,1]$. The program has the \textbf{dependence graph} $G(\mathbb{P}) := G(\textbf{P})$ and is \textbf{stratified} if the logic program $\textbf{P}$ is.

\begin{example}
    Assigning to the abducibles in Example \ref{example - non-stratified abductive logic program} the probabilities $\pi (ha) = \frac{2}{3}$ and $\pi (hk) = \frac{3}{4}$ yields a ProbLog program~$\mathbb{P} := (\mathcal{P}, \mathfrak{A})$.
    \label{example - ProbLog program}
\end{example}

Baral et al.~\cite{P-log} apply the maximum entropy principle (Principle~\ref{principle - maximum entropy}) to extend the distribution of subprograms to the corresponding models thereby obtaining the P-log semantics.

Let $\mathbb{P} := (\textbf{P}, \mathfrak{A}, \pi)$ be a ProbLog program. A \textbf{choice}~${\epsilon \subseteq \mathfrak{A}}$ is a subset of error terms. It is \textbf{consistent} if the logic program~$\textbf{P} \cup \epsilon$ has at least one stable model. 
Interpreting the probabilities $\pi(u)$ as an independent chance for~${u \in \mathfrak{A}}$ to be true, a choice~$\epsilon$ has the \textbf{probability}
\( \displaystyle
        \pi (\epsilon) := 
        \prod_{ u \in \epsilon} \pi (u) \cdot 
        \prod_{ u \in \mathfrak{A} \setminus  \epsilon} (1 - \pi (u)).
\)

\begin{remark}
    If a choice $\epsilon$ is not consistent, the logic program $\textbf{P} \cup \epsilon$ has no  model. Hence, one would conclude against the causal direction that $\epsilon$ cannot occur, which contradicts Principle \ref{principle - consistency with logical reasoning}.
    \label{remark - consistency}
\end{remark}

From now on, it is assumed that all choices $\epsilon$ of $\mathbb{P}$ are consistent. Then, the program $\textbf{P} \cup \epsilon$ has~${n \in \mathbb{N}_{\geq 1}}$ models~$\omega_1, \dots, \omega_n$ and, by the maximum entropy principle (Principle~\ref{principle - maximum entropy}), each world~$\omega_i$ receives the \textbf{probability}
    \( 
       \pi_{\textbf{P}}^{\text{P-log}} (\omega_i) := n^{-1} \cdot \pi (\epsilon).
    \)
Setting $\pi_{\mathbb{P}}^{\text{P-log}} (\omega) := 0$ for any~$\omega$ that is not a model of some~$\textbf{P} \cup \epsilon$ for a choice~$\epsilon$, the P-log semantics $\pi_{\textbf{P}}^{\text{P-log}}$ defines a probability distribution over $\mathfrak{P}$-structures.

\begin{example}
     The ProbLog program~$\mathbb{P} := (\textbf{P}, \mathfrak{A}, \pi)$ from Example~\ref{example - ProbLog program} has models~$\omega_1$-$\omega_5$ from Example \ref{example - stable models of abductive logic program} with probabilities $\pi_{\mathbb{P}}^{P-log} (\omega_1) = \frac{1}{12}$, $\pi_{\mathbb{P}}^{P-log} (\omega_2) := \frac{1}{6}$, $\pi_{\mathbb{P}}^{P-log}(\omega_3) = \frac{1}{4}$, and ${\pi_{\mathbb{P}}^{P-log} (\omega_4) = \pi_{\mathbb{P}}^{P-log} (\omega_5) = \frac{1}{4}}$. 
    \label{example - semantics of ProbLog programs}
\end{example}

\section{Results}
\label{sec: Results}

To analyze the causal expressiveness of \mbox{$\LP^{\MLN}$} and ProbLog programs, we summarize these frameworks as \emph{causal systems}, analogous to the (maximum entropy) causal systems of Eelink et al.~\cite{EelinkRW25}. This yields a causal interpretation and a corresponding notion of intervention for these programs.

\begin{definition}[Causal Systems]
A \textbf{causal system} ${\mathbf{CS} := (\mathbb{P}, \mathfrak{A}, \Phi)}$ consists of an $\LP^{\MLN}$ program $\mathbb{P}$, a set of abducibles ${\mathfrak{A} \subseteq \mathfrak{P}}$, and a LogLinear model $\Phi$ over $\mathfrak{A}$ such that no abducible $u \in \mathfrak{A}$ occurs as the head of a clause in $\mathbb{P}$. Arguing as in Remark~\ref{remark - consistency}, we further assert that $\mathbb{P} \cup \{(+\infty, p) \mid p \in \epsilon  \}$ has at least one stable model for each \textbf{choice} $\epsilon \subseteq \mathfrak{A}$.

The weight $w$ of a clause $(w,C) \in \mathbb{P}$ represents the degree of certainty in necessity (Principle~\ref{principle - Aquinas}). Hence, the LogLinear model $\constraint(\mathbb{P})$ encodes beliefs in necessity (Principle~\ref{principle - Aquinas}) given choices $\epsilon$.

\begin{formalization}[Natural Necessity]
Together with maximum entropy (Principle~\ref{principle - maximum entropy}) and Parametrization~\ref{parametrization - maximum entropy model}, necessity (Principle~\ref{principle - Aquinas}) implies that a world $\omega$ occurs with probability \\
\[{
\pi_{\mathbb{P}}^{\text{necessity}}(\omega) :=
\pi_{\constraint(\mathbb{P})}^{\MLN}(\omega \mid \omega \cap \mathfrak{A})
\cdot
\pi_{\Phi}^{\MLN}(\omega \cap \mathfrak{A}).
}\]
\label{formalization - natural necessity - CS}
\end{formalization}

The \textbf{explanatory content} of $\mathbf{CS}$ for a world $\omega$ is the abductive logic program ${\mathbf{CS} \vert_{\omega} := (\mathbb{P} \vert_{\omega}, \mathfrak{A})}$, and sufficiency (Principle~\ref{principle - Leibniz}) is formalized as follows.

\begin{formalization}[Causal Sufficiency]
A world $\omega$ satisfies sufficiency (Principle~\ref{principle - Leibniz}) with respect to $\mathbf{CS}$ if~$\omega$ is a model of $\mathbf{CS} \vert_{\omega}$. The event that $\mathbf{CS}$ satisfies sufficiency (Principle~\ref{principle - Leibniz}) is then given by the set of all such worlds $\sufficient(\mathbf{CS})$.
\label{formalization - causal sufficiency - CS}
\end{formalization}

To represent an intervention that forces atoms in $\mathbf{I} \subseteq \mathfrak{P} \setminus \mathfrak{A}$ to attain values according to the assignment~$\mathbf{i}$, the \textbf{modified system}
${\mathbf{CS}_{\mathbf{i}} := (\mathbb{P}_{\mathbf{i}}, \mathfrak{A}, \Phi)}$
is obtained from $\mathbf{CS}$ by the following modifications:
\begin{enumerate}
    \item Remove all clauses $(w,C)$ from $\mathbb{P}$ for which $\mathrm{head}(C) \in \mathbf{I}$.
    \item Add a fact $(\infty, p)$ to $\mathbb{P}_{\mathbf{i}}$ whenever $p \in \mathbf{i}$.
\end{enumerate}

A \textbf{semantics} is a partial map $\pi_{\cdot}$ that assigns to some causal systems $\mathbf{CS}$ a distribution $\pi_{\mathbf{CS}}$ on~$\mathfrak{P}$. 
\label{definition - causal systems}
\end{definition}

We identify $\LP^{\MLN}$ programs $\mathbb{P}$ with the causal systems 
$\mathbf{CS}(\mathbb{P}) := (\mathbb{P}, \emptyset, \emptyset)$ and  ProbLog programs ${\mathbb{P} := (\textbf{P}, \mathfrak{A}, \pi)}$ with the causal systems $\mathbf{CS} := (\textbf{P}, \mathfrak{A}, \Phi)$, where the logic program $\textbf{P}$ is understood as the $\LP^{\MLN}$ program $\{ (\infty, C) \mid C \in \textbf{P} \}$ and 
$\Phi := \{ (\ln(\pi(u)), u), (\ln(1 - \pi(u)), \neg u) \mid u \in \mathfrak{A} \}$.

Under Formalizations~\ref{formalization - natural necessity - CS} and~\ref{formalization - causal sufficiency - CS}, a calculation shows that the $\LP^{\MLN}$~\cite{LPMLN} and P-log semantics~\cite{P-log} compute the probability that worlds satisfying sufficiency (Principle~\ref{principle - Leibniz}) also satisfy necessity (Principle~\ref{principle - Aquinas}). Consequently, they straightforwardly generalize the causally founded worlds of Rückschloß and Weitkämper~\cite{EelinkRW25} under the maximum entropy principle (Principle \ref{principle - maximum entropy}) and Parametrization \ref{parametrization - maximum entropy model}.

\begin{proposition}[Consistency with Necessity and Sufficiency]
    \begin{enumerate}
    \item[] 
        \item[] For every $\LP^{\MLN}$ program $\mathbb{P}$ it holds that
    \(
    \pi_{\mathbb{P}}^{\LP^{\MLN}} (\cdot) := \pi_{\mathbb{P}}^{\text{necessity}} (\cdot \mid \sufficient (\mathbb{P})).
    \)
    \item[] For every ProbLog program $\mathbb{P} := (\textbf{P}, \mathfrak{A}, \pi)$ it holds that
    \(
    \pi_{\mathbb{P}}^{\text{P-log}} (\cdot) := \pi_{\mathbb{P}}^{\text{necessity}} (\cdot  \mid \sufficient (\mathbb{P})).\quad\square
    \)
    \end{enumerate}
    \label{proposition - analysis of semantics}
\end{proposition}

Following Eelink et al.~\cite{EelinkRW25}, a causal system must additionally satisfy causal irrelevance (Principle~\ref{principle - causal irrelevance}) in order to yield scientific knowledge. We generalize Formalization~\ref{formalization - causal irrelevance - BN} of causal irrelevance as follows.

\begin{formalization}[Causal Irrelevance]
Fix a causal system $\textbf{CS} := (\mathbb{P}, \mathfrak{A}, \Phi)$ over an alphabet $\mathfrak{P}$.  
For a subset of propositions $S \subseteq \mathfrak{P}$, let $\mathfrak{P}^{<S}$ denote the set of all propositions $q \notin S$ that are non-descendants of every proposition in $S$ in the dependency graph $G(\textbf{CS})$.  

Let 
$\mathbb{P}^{<S} := \{(w,C) \in \mathbb{P} \mid \head(C) \in \mathfrak{P}^{<S}\}$ 
be the program obtained from $\mathbb{P}$ by deleting all clauses whose head is not in $\mathfrak{P}^{<S}$, and let 
$\textbf{CS}^{<S} := (\mathbb{P}^{<S}, \mathfrak{A}, \Phi)$ 
be the corresponding causal system.

Then Principle~\ref{principle - causal irrelevance} holds for $\textbf{CS}$ under the semantics $\pi_{\cdot}$ if, for all $S \subseteq \mathfrak{P}$, the systems $\textbf{CS}$ and $\textbf{CS}^{<S}$ induce the same marginal distribution on $\mathfrak{P}^{<S}$.
\label{formalization - causal irrelevance - CS}
\end{formalization}

Causal irrelevance (Principle~\ref{principle - causal irrelevance}), as stated in Formalization~\ref{formalization - causal irrelevance - CS}, implies non-interference (Principle~\ref{principle - non-interference}).

\begin{proposition}
Assume a causal system $\mathbf{CS} := (\mathbb{P}, \mathfrak{A}, \Phi)$ over the alphabet $\mathfrak{P}$ satisfies causal irrelevance (Principle~\ref{principle - causal irrelevance}) with respect to a semantics $\pi_{\cdot}$ as stated in Formalization~\ref{formalization - causal irrelevance - CS}.  

Choose a subset $\mathbf{I} \subseteq \mathfrak{P} \setminus \mathfrak{A}$ together with a value assignment $\mathbf{i}$. Then, for every value assignment $\omega^{<\mathbf{I}}$ on $\mathfrak{P}^{< \textbf{I}}$, we have
\(
\pi_{\mathbf{CS}_{\mathbf{i}}}(\omega^{<\mathbf{I}}) = \pi_{\mathbf{CS}}(\omega^{<\mathbf{I}}).
\)
\label{proposition - causal irrelevance and intervention}
\end{proposition}

\begin{proof}
    This is a direct consequence from Formalization~\ref{formalization - causal irrelevance - CS} as $\textbf{CS}^{< \textbf{I}} = \textbf{CS}_{\textbf{i}}^{< \textbf{I}}$.
\end{proof}

Our core observation is that the semantics of $\LP^{\MLN}$ programs~\cite{LPMLN} and the P-log semantics of Baral et al.~\cite{P-log} do not satisfy Formalization~\ref{formalization - causal irrelevance - CS} of causal irrelevance (Principle~\ref{principle - causal irrelevance}). Consequently, counterintuitive results arise when reasoning about external interventions, since non-interference (Principle~\ref{principle - non-interference}) is violated.

\begin{example}
Let us combine Examples~\ref{example - intervention in abductive logic programs} and~\ref{example - ProbLog program}. The resulting situation is modeled by the ProbLog program $\mathbb{P} := (\mathcal{P}, \pi)$ obtained from the program $\mathcal{P}$ of Example~\ref{example - intervention in abductive logic programs} by setting $\pi(ha) := \frac{2}{3}$ and ${\pi(hk) := \frac{3}{4}}$ as in Example~\ref{example - ProbLog program}. It has the stable models ${\omega_1 = \emptyset}$, ${\omega_2 = \{ hk, k \}}$, ${\omega_3 = \{ ha, a, l \}}$, ${\omega_4 = \{ ha, a, m \}}$, ${\omega_5 = \{ ha, hk, k \}}$, ${\omega_6 = \{ ha, hk, a, l \}}$, and ${\omega_7 = \{ ha, hk, a, m \}}$ occurring with probabilities ${\pi_{\mathbb{P}}^{\text{P-log}}(\omega_1) = \frac{1}{12}}$, ${\pi_{\mathbb{P}}^{\text{P-log}}(\omega_2) = \frac{1}{4}}$, ${\pi_{\mathbb{P}}^{\text{P-log}} (\omega_3) = \pi_{\mathbb{P}}^{\text{P-log}}(\omega_4) = \frac{1}{12}}$, and ${\pi_{\mathbb{P}}^{\text{P-log}}(\omega_5) = \pi_{\mathbb{P}}^{\text{P-log}}(\omega_6) = \pi_{\mathbb{P}}^{\text{P-log}}(\omega_7) = \frac{1}{6}}$. 

Anna answers the door with probability $\pi_{\mathbb{P}}^{\text{P-log}}(a) = \frac{1}{2}$. In Examples~\ref{example - ProbLog program} and~\ref{example - semantics of ProbLog programs}, she answers the door with probability $\frac{5}{12}$, thereby violating causal irrelevance (Principle~\ref{principle - causal irrelevance}) under Formalization~\ref{formalization - causal irrelevance - CS}.

If Martina decides not to eat any chocolate, this leads to the modified program $\mathbb{P}_{\neg m} := (\mathcal{P}_{\neg m}, \pi)$, where $\mathcal{P}_{\neg m}$ is as in Example~\ref{example - intervention in abductive logic programs}. We then obtain $\pi_{\mathbb{P}_{\neg m}}^{\text{P-log}}(a) = \frac{5}{12} \neq \frac{1}{2}$, thereby violating non-interference (Principle~\ref{principle - non-interference}), since $a$ lies below $m$ in the causal order induced by the dependence graph $G(\mathbb{P})$.
\label{example - counterexample ProbLog}
\end{example}

\begin{example}
Similar to Example~\ref{example - counterexample ProbLog}, we model the situation from Examples~\ref{example - intervention in abductive logic programs} and~\ref{example - LPMLN programs} with the $\LP^{\MLN}$ program
\[
\mathbb{P} := \{ (\infty, m \Leftarrow a, \neg l),\; (\infty, l \Leftarrow a, \neg m),\; (\ln 2, k \Leftarrow \neg a),\; (\ln 2, a \Leftarrow \neg k) \}.
\]
It has the stable models $\omega_1 = \emptyset$, $\omega_2 = \{ k \}$, $\omega_3 = \{ a,l \}$, and $\omega_4 = \{ a,m \}$ with probabilities
${\pi_{\mathbb{P}}^{\LP^{\MLN}}(\omega_1) = \frac{1}{13}}$ and
$\pi_{\mathbb{P}}^{\LP^{\MLN}}(\omega_2) = \pi_{\mathbb{P}}^{\LP^{\MLN}}(\omega_3) = \pi_{\mathbb{P}}^{\LP^{\MLN}}(\omega_4) = \frac{4}{13}$.

Hence, Anna answers the door with probability
$\pi_{\mathbb{P}}^{\LP^{\MLN}}(a) = \frac{8}{13}$.
By contrast, in Example~\ref{example - LPMLN programs} she answers the door with probability $\frac{4}{9}$, violating causal irrelevance (Principle~\ref{principle - causal irrelevance}) under Formalization~\ref{formalization - causal irrelevance - CS}.

If Martina decides not to eat any chocolate, the modified program
\[
\mathbb{P}_{\neg m} := \{ \xcancel{(\infty, m \Leftarrow a, \neg l)},\; (\infty, l \Leftarrow a, \neg m),\; (\ln 2, k \Leftarrow \neg a),\; (\ln 2, a \Leftarrow \neg k) \}
\]
yields $\pi_{\mathbb{P}_{\neg m}}^{\LP^{\MLN}}(a) = \frac{4}{9} \neq \frac{8}{13}$, thus violating non-interference (Principle~\ref{principle - non-interference}).
\end{example}

However,  for stratified ProbLog programs, we find that the P-log semantics of Baral et al.~\cite{PLPconterfactuals} satisfies Formalization \ref{formalization - causal irrelevance - CS} of causal irrelevance (Principle \ref{principle - causal irrelevance}).

\begin{theorem}[Causal Irrelevance for Stratified ProbLog Programs]
    If $\mathbb{P} := (\textbf{P}, \mathfrak{A}, \Phi)$ is a stratified ProbLog program, one finds that~$\pi_{\cdot}^{\text{P-log}}$ satisfies Formalization \ref{formalization - causal irrelevance - CS} of Principle \ref{principle - causal irrelevance}.
    \label{theorem - Consistency for Stratified ProbLog Programs}
\end{theorem}

\begin{proof}
    Let $\epsilon \subseteq \mathfrak{A}$ be a choice and $S \subseteq \mathfrak{P} \setminus \mathfrak{A}$. Since $\textbf{P}$ is stratified, $\textbf{P} \cup \epsilon$ has a unique model $\omega$ and~${\omega \cap \mathfrak{P}^{<S}}$ is the unique stable model of $\textbf{P}^{<S}$. Hence, we find $\pi_{\mathbb{P}}^{\text{P-log}} = \pi (\epsilon) = \pi_{\mathbb{P}^{<S}}^{\text{P-log}}$.
\end{proof}

Note that the analogue of Theorem \ref{theorem - Consistency for Stratified ProbLog Programs} fails for the $\LP^{\MLN}$ semantics~\cite{LPMLN}.

\begin{example}
Consider the $\LP^{\MLN}$ program $\mathbb{P} := \{ (\ln 2, p),\; (\ln 2, q \Leftarrow \neg p) \}$ and set $S := \{ q \}$. We find that $\mathbb{P}^{<S} = \{ (\ln 2, p) \}$. A calculation now shows that
\(
\pi_{\mathbb{P}}^{\LP^{\MLN}}(p) = \frac{4}{7} \neq \frac{2}{3} = \pi_{\mathbb{P}^{<S}}^{\LP^{\MLN}}(p).
\)
\label{example - counterexample LPMLN}
\end{example}

Let $\textbf{CS} := (\mathbb{P}, \mathfrak{A}, \Phi)$ be an arbitrary causal system over an alphabet $\mathfrak{P}$. To obtain a semantics adhering to causal irrelevance (Principle~\ref{principle - causal irrelevance}) in the general case, we follow Williamson~\cite{Williamson2001} and Eelink et al.~\mbox{\cite{EelinkRW25}}. They interpret Principle~\ref{principle - causal irrelevance} as stated in Formalization~\ref{formalization - causal irrelevance - CS} and Principle~\ref{principle - maximum entropy} as requiring entropy to be maximized greedily along the given causal order. This leads to the following semantics for causal systems, which, as Eelink et al.~\mbox{\cite{EelinkRW25}} argue, satisfies Principles~\ref{principle - maximum entropy}--\ref{principle - causal irrelevance} under Parametrization~\ref{parametrization - maximum entropy model}.

\begin{definition}[Causal Semantics]
Recall that two nodes \( p \) and \( q \) of a directed graph \({ G := (V,E) }\) are 
\textbf{strongly connected}, written~\( p \sim q \), if there exist directed paths 
from \( p \) to \( q \) and from~\( q \) to \( p \) in~\( G \). Strong connectedness 
(\(\sim\))/2 is an equivalence relation, and the equivalence classes \( [p] \in V/\sim \) 
are called the \textbf{strongly connected components} of \( G \). Lastly, the resulting 
\textbf{factor graph} \( G/\sim := (V/\sim, E/\sim) \) is a directed acyclic graph, 
where \({ E/\sim := \{ ([p],[q]) \in (V/\sim)^2 \mid (p,q) \in E \} }\).

Let $G$ be the factor graph on the strongly connected components of the dependence graph $G(\textbf{CS})$. To each component~\( {V \in \textbf{V}} \), we associate a random variable ranging over all assignments
${
v : V \to \{\top, \bot\}.
}$

Let \( v \) be a value of a component \( V \in \textbf{V} \) such that~\({ V \cap \mathfrak{A} = \emptyset }\), and~\( \pa(V) \) a value assignment to \( \Pa(V) \). Define
${
    \mathbb{P} \vert V := \{ (w, C) \in \mathbb{P} \mid \head (C) \in V \}
}$ and ${
    \mathfrak{A}^{\Pa(V)} := \{ p \in \mathfrak{P} \mid p \in W,~ W \in \Pa(V) \}
}$.

Consider the system $\mathbf{CS}^{V} := (\mathbb{P} \vert V, \mathfrak{A}^{\Pa(V)}, \emptyset)$. Since all propositions $p \in V$ lie in a causal cycle, Eelink et al.~\cite{EelinkRW25} argue that \emph{causal irrelevance} (Principle~\ref{principle - causal irrelevance}), imposes no constraint. According to Formalizations~\ref{formalization - natural necessity - CS} and~\ref{formalization - causal sufficiency - CS} and Proposition~\ref{proposition - analysis of semantics}, necessity (Principle~\ref{principle - Aquinas}) and sufficiency (Principle~\ref{principle - Leibniz}) imply that there exists a causal explanation of $v$ with premises $\pa(V)$ with probability
\begin{align}
   & \pi_{\mathbf{CS}}^{\emph{causal}}(\mathbf{v} \mid \pa(V)) := \pi_{\mathbb{P} \vert V}(\mathbf{v} \mid \pa(V)) := \label{equation - demonstration} \pi_{\constraint(\mathbb{P} \vert V)}\bigl(\mathbf{v} \mid \pa(V), \sufficient(\LP(\mathbb{P} \vert V)) \bigr).
\end{align}
   
Applying causal irrelevance from Principle~\ref{principle - causal irrelevance} and greedily maximizing entropy from Principle~\ref{principle - maximum entropy} yields the causal semantics for causal systems.


The \textbf{causal structure} $\graph (\textbf{CS})$ of $\textbf{CS}$ results from $G$ by replacing all nodes $V := \{ p \}$ for $p \in \mathfrak{A}$ and all edges $V \rightarrow W$ with one node $S := \mathfrak{A}$ and the edges $S \rightarrow W$. The system $\textbf{CS}$ then assumes the \textbf{inexplainable knowledge} that a choice $\epsilon \subseteq \mathfrak{A}$ occurs with probability: 
    \begin{equation}
     \pi_{\textbf{CS}}^{\text{causal}}(\epsilon) := 
     \pi_{\Phi} (\epsilon).     
    \label{equation - indemonstrable knowledge}
    \end{equation}

Let $\textbf{BN}(\textbf{CS})$ be the Bayesian network that is given by the causal structure $\graph (\textbf{CS})$ and Probabilities (\ref{equation - demonstration}) and~(\ref{equation - indemonstrable knowledge}). The \textbf{causal semantics} of $\textbf{CS}$
is then given by setting
$
\pi_{\textbf{CS}}^{\text{causal}} (\omega) := \pi_{\textbf{BN}(\textbf{CS})}^{BN}(\omega)
$ for each world~$\omega$.
\label{definition - causal semantics}
\end{definition}

\begin{example}
    In Example~\ref{example - counterexample ProbLog}, $\textbf{BN}(\mathbb{P})$ is given by the causal structure 
    $\{ha, hk\} \rightarrow \{a,k\} \rightarrow \{l,m\}$ 
    and probabilities $\pi(ha, hk) = \frac{1}{2}$, \ldots, $\pi(\neg ha, \neg hk) = \frac{1}{12}$, 
    $\pi(a, \neg k \mid ha, hk) = \frac{1}{2}$, \ldots, $\pi(\neg a, \neg k \mid \neg ha, \neg hk) = 1$, 
    and $\pi(l, \neg m \mid k, \neg a) = 0$, \ldots, $\pi(l, \neg m \mid \neg k, a) = \frac{1}{2}$.

    We obtain the probabilities in Example~\ref{example - counterexample ProbLog}, except 
    $\pi_{\mathbb{P}}^{\text{causal}}(\omega_5) = \frac{1}{4}$ and 
    ${\pi_{\mathbb{P}}^{\text{causal}}(\omega_6) = \pi_{\mathbb{P}}^{\text{causal}}(\omega_7) = \frac{1}{8}}$. 
    Hence, Anna answers the door with probability $\frac{5}{12}$, as in Examples~\ref{example - ProbLog program} and~\ref{example - semantics of ProbLog programs}.
    \label{example - causal semantics for ProbLog programs}
\end{example}

\begin{example}
    In Example~\ref{example - counterexample LPMLN}, $\textbf{BN}(\mathbb{P})$ is given by the causal structure $p \rightarrow q$ with probabilities $\pi(p) = \frac{2}{3}$, $\pi(q \mid p) = 0$, and $\pi(q \mid \neg p) = \frac{2}{3}$. Consequently,
    \(
    \pi_{\mathbb{P}}^{\text{causal}}(p) = \frac{2}{3} = \pi_{\mathbb{P}^{<S}}^{\text{causal}}(p),
    \)
    as desired.
    \label{example - causal semantics for LPMLN programs}
\end{example}

The causal semantics consistently generalizes the causal interpretation of stratified ProbLog programs in Theorem \ref{theorem - Consistency for Stratified ProbLog Programs} thereby maintaining causal irrelevance (Principle~\ref{principle - causal irrelevance}) as stated in Formalization \ref{formalization - causal irrelevance - CS}.

\begin{theorem}[Consistency for Stratified ProbLog Programs]
    If $\mathbb{P} := (\textbf{P}, \mathfrak{A}, \Phi)$ is a stratified ProbLog program, one finds that~${\pi_{\mathbb{P}}^{\text{P-log}}  = \pi_{\mathbb{P}}^{\text{causal}} }$.
    \label{theorem - Consistency for Stratified ProbLog Programs - 1}
\end{theorem}

\begin{proof}
    Let $\epsilon \subseteq \mathfrak{A}$ be a choice. Since $\mathbb{P}$ is stratified, there is a unique stable model $\omega$ of $\textbf{P} \cup \epsilon$. For every random variable $V$ in the Bayesian network $\textbf{BN}(\mathbb{P})$, we have $\pi(v \mid \pa(V) \cap \omega) = 1$ if $v = \omega \cap V$, and $\pi(v \mid \pa(V) \cap \omega) = 0$ otherwise. Consequently, $\pi_{\mathbb{P}}^{\text{causal}}(\omega) = \pi(\epsilon) = \pi_{\mathbb{P}}^{\text{P-log}}(\omega)$.
\end{proof}

\begin{theorem}[Causal Irrelevance for Causal Systems]
    The causal semantics $\pi_{\cdot}^{\text{causal}}$ satisfies causal irrelevance as stated in Formalization \ref{formalization - causal irrelevance - CS}. 
    \label{theorem - causal irrelevance}
\end{theorem}

\begin{proof}
    This follows by construction from causal irrelevance in Bayesian networks (Formalization \ref{formalization - causal irrelevance - BN}).
\end{proof}

\section{Implementation}
\label{sec: Implementation}
The causal semantics proposed here is implemented as part of the PLP-BN suite of tools connecting probabilistic logic programs and Bayesian Networks (\url{https://github.com/weitkaemper/plpbn-tools}). 
Written in Logtalk \cite{Logtalk}, this suite of tools represents formalisms like ProbLog programs, $\LP^{\MLN}$ programs and Bayesian networks as object protocols implemented by concrete objects; a parser is provided for reading programs in ProbLog notation. 
It interfaces with a utility library providing Logtalk implementations of factor and variable elimination~\mbox{\cite[Chapters 6 and 7]{Darwiche2009}}, i.e., inference algorithms for discrete Bayesian networks (\url{https://github.com/weitkaemper/categorical-bn-logtalk}).

In this scheme, a Bayesian network on a graph $G$ of discrete random variables exposes a predicate \texttt{cpt/4} taking as input a node, a value for that node and a specification of its parents, and returning a conditional probability between 0 and 1. 

In our implementation of the present semantics, we use the answer set solver Clingo \cite{Clingo} (via a SWI-Prolog interface due to Jan Wielemaker) to compute the answer sets for a single strongly connected component. 
This can be done directly in the case of general cyclic ProbLog programs, while in the case of $\LP^{\MLN}$ programs we use the penalty encoding proven correct by Lee et al.\ \cite{LPMLN2ASP}.  

\section{Conclusion}
Examples~\ref{example - counterexample ProbLog}--\ref{example - counterexample LPMLN} show that the established semantics for $\LP^{\MLN}$~\cite{LPMLN} and ProbLog~\cite{ProbLog} violate causal irrelevance, i.e., beliefs may change when additional unobserved effects are considered. Proposition~\ref{proposition - causal irrelevance and intervention} relates this phenomenon to interventions, further explaining the counterintuitive results in these examples. 
Theorem~\ref{theorem - Consistency for Stratified ProbLog Programs} shows that stratified ProbLog programs satisfy causal irrelevance under P-log semantics~\cite{P-log}. 

We introduce a new formal causal semantics for general $\LP^{\MLN}$ and ProbLog programs together with an implementation that preserves causal irrelevance and yields improved behavior when processing interventions. Theorems~\ref{theorem - Consistency for Stratified ProbLog Programs - 1} and \ref{theorem - causal irrelevance} show that this semantics both extends the causal interpretation of stratified ProbLog programs and satisfies causal irrelevance.

Future work should extend the proposed semantics to non-propositional probabilistic logic programming and study relational representations of interventions therein. Furthermore, integrating the present atemporal notion of causal reasoning with the temporal framework of Vennekens et al.~\cite{cplogic} is a promising direction for future research.

\bibliographystyle{eptcs}
\bibliography{generic}
\end{document}